\documentclass[runningheads]{llncs}

\usepackage{graphicx}
\usepackage[bb=boondox]{mathalfa}
\usepackage{amssymb}
\usepackage{amsmath}
\usepackage{xspace}
\usepackage[caption=false]{subfig}
\usepackage{mathtools} 
\usepackage{array,multirow}
\usepackage{paralist}
\usepackage{tikz}
\usetikzlibrary{patterns}
\usetikzlibrary{decorations.pathreplacing}
\usetikzlibrary{arrows}
\usepackage[bb=boondox]{mathalfa}
\usepackage[%
  textwidth=1.5cm%
 ]{todonotes}
 \usepackage{bussproofs}
 
\newcommand{\nat}{\mathbb{N}}

\newcommand{\myI}{\mathcal{I}}

\newcommand{\pomLogic}{\mathcal{L}^{min}_{pOM}\xspace}
\newcommand{\conicHull}[1]{ch(#1)}

\newcommand{\XCone}{a}

\newcommand{\YCone}{b}

\newcommand{\oneg}[1]{{#1}^{\bot}}
\newcommand{\odneg}[1]{{#1}^{\bot\bot}}

\newcommand{\gbund}{\mathrel{\&}}
\newcommand{\gboder}{\mathrel{\vee}}
\newcommand{\gbnicht}{\mathop{\thicksim}}

\newcommand{\lmeet}{\wedge}
\newcommand{\ljoin}{\vee}
\newcommand{\lneg}[1]{{#1}^{\bot}}
\newcommand{\ltop}{\mathbb{1}}
\newcommand{\lbot}{\mathbb{0}}
\newcommand{\lunder}{\leq}
\newcommand{\lunderp}{<}
\newcommand{\lcover}{\mathrel{<:}}

\newcommand{\cund}{\sqcap}
\newcommand{\coder}{\sqcup}
\newcommand{\cnicht}[1]{{#1}^{\circ}}

\newcommand{\cunder}{\sqsubseteq}
\newcommand{\ctop}{\raisebox{-.075em}{\rotatebox[origin=c]{270}{$\vDash$}}}
\newcommand{\cbot}{\raisebox{-.075em}{\rotatebox[origin=c]{90}{$\vDash$}}}

\newcommand{\aCone}{a}
\newcommand{\bCone}{b}
\newcommand{\cCone}{c}

\newcommand{\minC}{$\mathit{MC}_{8}$\xspace}

\newcommand{\cmtp}{(pOM)\xspace}

\newcommand{\wllj}{(wLLJ)\xspace}

\newcommand{\raus}[1]{}

\newcommand{\AteilB}{(PM1)\xspace}
\newcommand{\nichtABunterC}{(PM2)\xspace}
\newcommand{\BCACC}{(PM3)\xspace}
\newcommand{\ABconclusion}{(C)\xspace}

\newenvironment{proofsketch}%
{
\noindent{\emph{Proof sketch.}} 
}%
{\hfill \qed\vspace{1ex}}

\begin{document}
\title{Orthologics for Cones}
\author{
	Mena Leemhuis\inst{1}\orcidID{0000-0003-1017-8921}
	\and
	{\"O}zg\"ur L.\ {{\"O}z{\c{c}}ep}\inst{1}
	\orcidID{0000-0001-7140-2574} 
\and
	Diedrich Wolter\inst{2}
	\orcidID{0000-0001-9185-0147}}


\institute{University of L\"ubeck, L\"ubeck, Germany\\ \email{mena.leemhuis@student.uni-luebeck.de \\ oezcep@ifis.uni-luebeck.de} \and University of Bamberg, Bamberg, Germany\\
	\email{diedrich.wolter@uni-bamberg.de}}

\maketitle    

\begin{abstract}
In applications that use knowledge representation (KR) techniques, in particular those that combine data-driven and logic methods, the domain of objects is not an abstract unstructured domain, but it exhibits a dedicated, deep structure of geometric objects.
One example is the class of convex sets used to model natural concepts in conceptual spaces, which also links via convex optimisation techniques to machine learning.
In this paper we study logics for such geometric structures. 
Using the machinery of lattice theory, we describe an extension of minimal orthologic with a partial modularity rule that holds for closed convex cones.
This logic combines a feasible data structure (exploiting convexity/conicity) with sufficient expressivity, including full orthonegation (exploiting conicity).
\end{abstract}

\section{Introduction\label{sect:introduction}}


In many applications of artificial intelligence (AI) the domain of objects has a dedicated, deep mathematical structure, which, usually,  is of a geometric kind.   
One example is the class of convex sets used to model natural concepts in conceptual spaces \cite{gaerdenfors00conceptual}.  
 Another related example, which is of relevance both for the knowledge representation (KR) and machine learning (ML) community, is that of knowledge graph embeddings \cite{wang17knowledgeEmbedding} where concepts and relations are  viewed as geometric objects or functions on a continuous space.  
A different example of KR-relevant structures consisting of geometric objects is that of closed subspaces in a Hilbert space. Closed subspaces can be used to model partial information \cite{hartonas16reasoning} regarding the states and measurements of particles on the micro-level \cite{birkhoff37logic}. 

In any of these and related examples,  qualitative reasoning over the geometric structures requires identifying a logic that captures the properties of the geometric objects.  In case of knowledge graph embeddings, say, the approach of  \cite{gutierrez-basulto18from} describes a class of logics (existential Datalog  fulfilling the quasi-chainedness property) that fits to geometric models of a specific kind,  namely those in which arbitrary relations are interpreted (again) as convex regions in Euclidean space. More concretely, the authors show that any logically consistent ontology over the specified datalog fragment has  a convex-region based geometric model. 
In the case of  closed subspace Hilbert spaces the logics that have been considered to be appropriate are quantum logics \cite{birkhoff37logic}.
 
Using the machinery of lattice theory \cite{graetzer11lattice}, this paper proposes the  propositional logic $\pomLogic$ to capture  useful properties of the class of closed convex cones. 
The class of closed convex cones should be of interest for various KR areas as it provides a good balance between computational feasibility (convexity) and 
expressivity (conicity). Indeed, convexity has been identified  as a useful property not only from a cognitive-linguistical perspective \cite{gaerdenfors00conceptual} but also from the viewpoint of computational feasibility---a case in point being the field of convex optimisation \cite{boyd04convex}.  
Conicity proves to be useful as it allows us to define a notion of negation in form of the polarity operation. This form of negation goes beyond the negation of 
\cite{gutierrez-basulto18from} and \cite{kulmanov19ele}, where negation is atomic and can hence  (only) represent integrity constraints and disjointness, but not covering constraints. 

Our proposed  logic $\pomLogic$ is an extension of minimal orthologic \cite{goldblatt74semantic}, a propositional non-distributive logic  with orthonegation, i.e., a negation that fulfils  antitonicity (contraposition), the intuitionistic absurdity principle (anything follows from a sentence stating $A$ and its negation) and double negation elimination. 
Such orthologics are closely related to ortholattices (the latter resulting from the former by the Lindenbaum-Tarski-construction). The proposed logic $\pomLogic$ contains an additional rule, called \cmtp,  that generalises the  orthomodularity property of minimal quantum logic. 
 
In earlier publications \cite{oezcep20cone,leemhuis20multi-label} we applied the idea of using cones for embedding knowledge graphs with background knowledge expressed in the semi-expressive description logic $\mathcal{ALC}$  and showed how to use cones for typical ML problem such as multi-label learning \cite{gibaja14multilabel}.        
In those publications we assumed the logic to be given in advance and as being distributive. This lead to severe restrictions on the overall configuration of cones---requiring them to be axis-aligned. 
As a consequence, a higher-dimensional space was required in models than would be necessary with arbitrarily positioned cones.

  In this paper, we drop the distributivity assumption and  investigate the non-distributive ``natural'' logics that hold for arbitrary configurations of cones. 

The  contributions of this  paper are the following: a theorem showing that closed convex cones fulfil the partial orthomodularity rule \cmtp;  a forbidden-subalgebra theorem showing that any ortholattice not fulfilling \cmtp must contain a particular minimal subortholattice, and a representation theorem characterising any logic fulfilling this additional rule. 

The rest of the paper is structured as follows: In Sect.\ \ref{sect:preliminaries} we set up the basic  lattice-theoretic, geometric, and logical machinery. 
In Section \ref{sect:coneLogic} we first give cone counterexamples to many prominent rules  discussed as potential weakenings of distributivity and then introduce the partial orthomodularity rule \cmtp and then prove, for the induced logic $\pomLogic$, the three results mentioned above.  
 Section~\ref{sect:relatedWork} discusses related work. 
 We finish by drawing conclusions.


\section{Preliminaries}\label{sect:preliminaries}
In this section  we give the main basic notions and results from lattice theory (Sect.~\ref{subsec:latticePreliminaries}), relevant geometrical notions in the context of cones (Sect.~\ref{subsec:latticePreliminaries}), and  the main bits of orthologics (Sect.~\ref{subsec:orthologicsPreliminaries}) as developed by Goldblatt \cite{goldblatt74semantic}. 

\subsection{Lattices\label{subsec:latticePreliminaries}}
\raus{A \emph{lattice} $(L, \lunder)$ is a structure with domain $L$ and a partial order $\lunder$ such that for any pair of elements $a,b \in L$ there is a smallest upper bound denoted $a \ljoin b$ and a largest lower bound denoted $a \lmeet b$.  As usual, $x \lunderp y$ means $x \lunder y$ and not $ y \lunder x$.  A \emph{bounded lattice}  $(L, \leq)$  contains a smallest element $\lbot$ and a largest element $\ltop$, i.e., elements such that for all $x \in L$ one has $ \lbot \lunder x \lunder \ltop $. An element $b$ covers an element $b$, for short $a \lcover b$ iff  for all $c$ with $a \leq c \leq b$ either $c = a$ or $c  = b$.} 

  
A lattice is called \emph{distributive} iff  for all $a,b,c \in L$: $a \lmeet (b \ljoin c) = (a \lmeet  b) \ljoin (a \lmeet  c)$ (and dually: $a \ljoin (b \lmeet  c) = (a \ljoin b) \lmeet  (a \ljoin c)$). 

The binary \emph{modularity relation}  on a lattice  \cite{fofanovaSemi-modular}, $(L,\leq$) is defined for all $a,b \in L$ as follows: $M(a,b) :\Leftrightarrow \forall c \leq a: a \wedge (b \vee c) = (a \wedge b) \vee c$.   
A pair $(a,b)$ is said to be a \emph{modular pair} iff $M(a,b)$ holds.  
 A lattice is called \emph{modular} iff $M(a,b)$ for all $a,b$. 
The modularity relation generalises the property of distributivity. 
 
  


 A lattice is called \emph{M-symmetric}  
 iff the modularity relation is symmetric, i.e.,  iff $M(a,b)$ entails $M(b,a)$.  
 
In a lattice an element $a^*$ is called a \emph{complement of} $a$ iff $a \lmeet a^* = \lbot$ and $a \ljoin a^* = \ltop$.  A lattice is said to be \emph{complemented (uniquely complemented)} iff each $a$ has a complement (has exactly one complement).   A bounded lattice $L$ is called an \emph{ortholattice} iff it has an orthocomplement $\oneg{\cdot}$, i.e., a function such that  for all $a,b \in L$ the following three conditions hold:   
\begin{itemize}
\item $a \lunder b$ entails $\oneg{b} \lunder \oneg{a}$ \hfill (antitonicity)
\item $\odneg{a} = a$ \hfill (double negation elimination)
\item $\lbot = a \lmeet \oneg{a}$ \hfill (intuitionistic absurdity)
\end{itemize}
Any ortholattice satisfies de Morgan's laws, i.e.,  for any $a,b \in L$ it holds that $\oneg{(a \lmeet b)} = \oneg{a} \ljoin \oneg{b}$    (and dually:  $\oneg{(a \ljoin b)} = \oneg{a} \lmeet \oneg{b}$). 

\raus{Roughly,  ortholattices can be understood as Boolean algebras without the distributivity rule. But there is a more fine-grained characterisation by MacNeille \cite{macneille37partially}.
 This characterisation  
states that  an algebraic structure is a Boolean algebra iff it is an ortholattice and fulfils the following additional axiom:
\begin{description}
\item[(*)] For all $a,b$: If for all $c$: $a \wedge b \leq c$, then $a \leq \oneg{b}$.
\end{description}  
The characterisation is such that it dispenses with the special elements $\ltop$ and $\lbot$. But  if the smallest element $\lbot$ is allowed, then this axiom can be expressed in the following form:  
\begin{description}\label{weakLLJ}
\item[\wllj] For all $a,b$:  If $a \wedge b \leq \lbot$, then $a \leq \oneg{b}$.
\end{description}  
 We call this rule \emph{weak Johansson's constructive contraposition} as it is a special case of Johansson's constructive contraposition named (LLJ) in \cite{hartonas16reasoning}.
 \begin{description}
\item[(LLJ)] For all $a,b,c$:  If $a \wedge b \leq c$, then $a \lmeet \oneg{c} \leq \oneg{b}$. 
\end{description}  
 When setting $c = \lbot$ in (LLJ) one immediateley gets \wllj.} 

An ortholattice is called \emph{orthomodular} iff one of the following equivalent conditions \cite[pp.\ 35--36]{redei98quantum} of orthomodularity holds: 
\begin{description}
\item[(OMr)] If $a \lunder b$ and $\oneg{a} \lunder c$ then $a \ljoin (b \lmeet c) = (a \ljoin b) \lmeet (a \ljoin c)$.\\
\mbox{\quad} \hfill (orthomodularity)
\item[(sOMr)]  If $a \lunder b$ then $b = a \ljoin (\oneg{a} \lmeet b)$.\\ 
\mbox{\quad} \hfill(short form orthomodularity)
\item[(dsOMr)]  If $b \lunder a$ then $b = a \lmeet (\oneg{a} \ljoin b)$.\\
\mbox{\quad} \hfill (dual short form of orthomodularity)
\end{description}
Clearly any modular ortholattice is also orthomodular.  

 
%
%
 \begin{figure}
\centerline{
\begin{tikzpicture}[scale=0.4]
\node (top) at (3,6) {$\ltop$};
\node (a) at (1,4) {$a$} edge (top);
\node (c) at (1,2) {$b$} edge (a);
\node (b) at (5,4) {$\oneg{b}$} edge (top);
\node (d) at (5,2) {$\oneg{a}$} edge (b);
\node (bot) at (3,0) {$\lbot$} edge (c) edge (d);
\end{tikzpicture}
\hfill
\begin{tikzpicture}[scale=0.7]
  \node[inner sep=0pt,outer sep=0pt] at (0,0) (ursprung){};
\draw (ursprung.45) -- (45:2);
\draw (ursprung.135) -- (135:2);
\fill[black!20!white]  (ursprung.45) -- (45:2)--(135:2)--cycle;
\draw node at (-2,1){$\aCone$};  
\draw (ursprung.60)--(60:1.6cm);
\draw (ursprung.120)--(120:1.6cm);
\fill[black!10!white] (ursprung.60)--(60:1.6cm)--(120:1.6cm)--cycle;
\draw node at (0,1){$\bCone$};  
\draw node at (-2,-0.5){$\bCone^{\circ}$};  
\fill[black!10!white] (ursprung.330)--(330:2.8cm)--(210:2.8cm)--cycle;
\fill[black!20!white] (ursprung.315) -- (315:2)--(225:2)--cycle;
\draw node at (0,-1){$\aCone^{\circ}$};  
\draw (ursprung.330) --(330:2.8cm);
\draw (ursprung.210) --(210:2.8cm);
\draw (ursprung.315) -- (315:2);
\draw (ursprung.225) -- (225:2);
\end{tikzpicture}
\hfill
		\begin{tikzpicture}[xscale=0.35, yscale = 0.35]
		\node (top) at (3,10) {$\ltop$};
		\node (C) at (0,7) {$c$} edge (top); 
		\node (nB) at (3,8) {$\oneg{b}$} edge (top); 
		\node (A) at (6,7) {$a$} edge (top); 
		\node (nX) at (3,6) {$\oneg{d}$} edge (nB); 
		\node (nA) at (0,3) {$\oneg{a}$} edge (nX) edge (C); 
		\node (X) at (3,4) {$d$} edge (C) edge (A); 
		\node (nC) at (6,3) {$\oneg{c}$} edge (nX) edge (A); 
		\node (B) at (3,2) {$b$} edge (X); 
		\node (bot) at (3,0) {$\lbot$} edge (nA) edge (B) edge (nC);    
		\end{tikzpicture}
}
\caption{Left:Ortholattice $O_{6}$; middle: Cone embedding of $O_{6}$; right: lattice \minC  
\label{fig:hexagonAndItsConeEmbedding}
}
\end{figure}
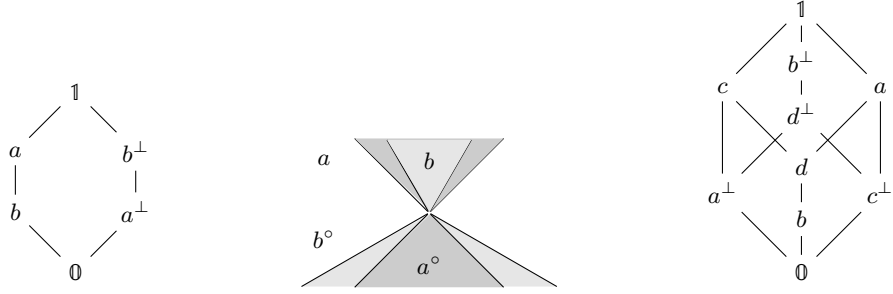

In lattice theory there are interesting characterisations of lattices fulfilling some given property/rule/axiom by so-called \emph{forbidden-subalgebra theorems} \cite[p.\ 134]{padmanabhan08axioms}. 
Consider the  Chinese lantern/hexagon lattice on the left hand side  of Fig.\ \ref{fig:hexagonAndItsConeEmbedding}, also known  under $O_{6}$  in the literature \cite{padmanabhan08axioms}. It is easy to check that the hexagon is indeed a lattice and orthocomplemented.
On the other hand this lattice is not orthomodular according to (dsOMr):  we have $b \lunder a$, but also   $b \neq a \lmeet (\oneg{a} \ljoin b) = a \lmeet \ltop = a$.   
 Now the forbidden-subalgebra theorem for ortholattices $L$ \cite[p.\ 134]{padmanabhan08axioms}  states that  $L$ is orthomodular iff it excludes $O_{6}$ as a sublattice. 
 In Sect.\ \ref{sect:coneLogic} we state a similar theorem for a generalisation of the orthomodularity rule. The structure used for this forbidden subalgebra theorem is the right figure in Fig.\ \ref{fig:hexagonAndItsConeEmbedding}.

\subsection{Cones \label{subsec:latticePreliminaries}}
We are going to consider geometric objects in finite dimensional Euclidean spaces $\mathbb{R}^n$ which are equipped  with a dot product $\langle \cdot, \cdot \rangle$. For any $x = (x_{1}, \dots, x_{n})$, $y = (y_{1}, \dots, y_{n}) \in \mathbb{R}^n$ the \emph{dot product} is defined as $\langle x,y \rangle = \Sigma_{1 \leq i \leq n} x_{i} \cdot y_{i}$. 
\raus{The dot product induces a norm for vectors $||x|| = \langle x, x\rangle$, and hence a metric $d(x,y) = ||x -y||$. Based on a metric $d$ one can define the open balls  $B_{\epsilon,x} = \{y \mid d(x,y) < \epsilon\}$, which in turn lead to the notion of \emph{open} sets $O$ (for each $x \in O$ there is a ball $B(\epsilon,x) \subseteq O$ contained in $O$)  and \emph{closed} sets (= complements of the open sets). }

A \emph{convex cone} $\XCone$ is a set such that from $x,y \in \XCone$ it follows that $\lambda x + \mu y \in \XCone$ for any $\lambda, \mu \in \mathbb{R}_{\geq 0}$.  We consider  convex cones that are closed in the canonical topology of $\mathbb{R}^n$ as defined above. 
One of the nice properties of closed convex cones is that they allow a polarity operation that takes the role of an orthocomplent. 
The \emph{polar cone} $\XCone^{\circ}$ for $\XCone$ is defined for Euclidean spaces with a dot product $\langle \cdot, \cdot \rangle$ as follows:  
\[\XCone^{\circ} = \{ x \in \mathbb{R}^n \mid \forall y \in \XCone: \langle x,y\rangle \leq 0 \}  \] 
Interpreting the dot product as a similarity measurement (as done in many ML scenarios)  the definition says that the polar cone contains those objects in the whole space that are not properly similar to any of the objects in $\XCone$.  

Now consider the subset-relation $\cunder = \subseteq$ 
on closed convex cones in $\mathbb{R}^n$ as a partial order. Closed convex cones are closed under  set intersection, so $\cap$ is a meet operator $\lmeet$ w.r.t.\ $\lunder$. Closed convex cones are not closed under set union. Instead they have to be closed up by the conic hull operator. The \emph{conic hull} of a set $\YCone$, for short $\conicHull{\YCone}$, is the smallest convex cone containing $\YCone$. So, we can define  join operation $\ljoin$  by $\XCone \ljoin \YCone = \conicHull{\XCone \cup \YCone}$. Considering $\mathbb{R}^n$ as the largest lattice element $\ltop$ and $\{0\}$, for $0 \in \mathbb{R}^n$, as the smallest lattice element $\lbot$ makes the resulting structure a bounded lattice.  

Using its definition, one can show that the polarity operator for  closed convex cones fulfils the properties of an orthocomplement. Hence the set of all closed convex cones (over $\mathbb{R}^n$) forms an ortholattice. As de Morgan's laws hold in any ortholattice, one gets in particular the following characterisation of the conic hull:   $\conicHull{\XCone \cup \YCone} = (\XCone^{\circ} \cap \YCone^{\circ})^{\circ}$. 
We denote the set of all closed convex cones in $\mathbb{R}^n$ by $\mathcal{C}_{n}$. 
\begin{proposition} 
For any $n \geq 1$, $\mathcal{C}_{n}$ is  an ortholattice. 
\end{proposition}

We use dedicated symbols for the signature of ortholattices when we talk about the ortholattice of closed convex cones:  $\cunder = \subseteq$ stands for the lattice order $\lunder$, $\cund = \cap$ stands for lattice meet $\lmeet$,      $\coder$ stands for lattice join $\ljoin$, and $\cnicht{}$ stands for orthocomplement $\lneg{}$, $\ctop = \mathbb{R}^n$ stands for the largest element $\ltop$ and $\cbot = \{0\}$ stands for the smallest element $\lbot$.\footnote{For disambiguation in this paper  we rely on symbols $\ctop$ and $\cbot$ which may be used with a different semantics in other papers.}

\subsection{Orthologics\label{subsec:orthologicsPreliminaries}}
   In this paper we are going to investigate  the class of all structures of closed convex cones over $\mathbb{R}^n$ (for any $n$) that make up an ortholattice. For this class we can expect to find interesting properties that can be  inferred in a calculus over some simple (non-distributive)  propositional logic. The reason is that this class is closed w.r.t.\ subortholattices and as such is almost a quasivariety \cite[p.\ 421]{graetzer11lattice} which can be described by implications of lattice identities.  To do so we rely on the general framework of  orthologics \cite{goldblatt74semantic}.  
   
 Let $P = \{P_{i}\mid i \in \nat\}$ be a set of proposition symbols and assume that we have logical symbols for binary conjunction $\gbund$, unary  negation $\gbnicht$, and binary disjunction $\gboder$. The set of propositional formulae $Fml(P)$ over $P$ is defined as usual. $A,B,C$ stand for propositional formulae in $Fml(P)$.  
 
 We consider  natural deduction calculi with a derivability relation $\vdash$.  We use the short notation $A \dashv \vdash B$ for  $A \vdash B$ and $B \vdash A$.  Moreover, for a finite set of formulae $ \Gamma = \{B_{1}, \dots, B_{n}\}$ the notation $\Gamma \vdash A$ is a shorthand for $B_{1} \gbund \dots \gbund B_{n} \vdash A$. 
The calculus of \emph{minimal orthologic} $Omin$ according to Goldblatt \cite{goldblatt74semantic} is given in Fig.\ \ref{fig:orthologic}.

\begin{figure}

\begin{tikzpicture}[xscale=0.5, yscale = 0.7]

\node[anchor= south west] at (0,5.2){Axioms};
\node[anchor= south west] at (11,5.2){Rules};
\node[anchor=north west] at (0,5){
\begin{tabular}{l}
$A \vdash A$  \quad $A \gbund B \vdash A$  \quad
 $A \gbund B \vdash B$\\ 
$A \dashv \vdash \gbnicht \gbnicht A$ \quad 
$A \gbund \gbnicht A \vdash B$\\
 $A \gboder B  \dashv \vdash \gbnicht ( \gbnicht A \gbund \gbnicht B)$
\end{tabular}
};
\node[anchor=north west] at (11,5){
\begin{tabular}{c}
$A \vdash B$,  $B \vdash C$\\ \hline
$A \vdash C$
\end{tabular}
};
\node[anchor=north west] at (16,5){
\begin{tabular}{c}
$A \vdash B$,  $A \vdash C$\\ \hline
$A \vdash B \gbund C$
\end{tabular}
};
\node[anchor=north west] at (21,5){
\begin{tabular}{c} $A \vdash B$\\ \hline
  $\gbnicht B \vdash \gbnicht A$
\end{tabular}
};

\end{tikzpicture}


\caption{Minimal Orthologic $Omin$}
\label{fig:orthologic}
\end{figure}

Any logic $\mathcal{L}$ containing the  rules of Fig.\ \ref{fig:orthologic} is called an \emph{orthologic}. 
For any orthologic the well-known Lindenbaum-Tarski construction leads to an ortholattice:  
The binary relation $\dashv \vdash$ can be shown to be an equivalence relation inducing for each formula $C$ an equivalence class $[C]$. Define operations $\lmeet$, $\ljoin$, $\oneg{}$ on the equivalence classes by  setting $ [C] \lmeet [D] = [C \gbund D]$, $  [C] \ljoin [D] = [C \gboder D]$ and $\oneg{[C]} = [\gbnicht C]$. These yield an ortholattice.

Goldblatt \cite{goldblatt74semantic} defines the semantics of orthologics 
based  on a structure $(X, \bot)$ called an \emph{orthoframe}. It consists of a  domain/carrier  $X$ and a binary \emph{orthogonality relation} $\bot \subseteq X \times X$, i.e., a relation that is  irreflexive and symmetric.\footnote{Note that in our considerations with the dot product $\langle \cdot, \cdot \rangle$ we have only an ``almost'' irreflexive orthogonality relation $x \bot y $ defined by  $\langle x,  y \rangle \leq 0$. But this is not a problem as the results of Goldblatt also hold for this almost reflexive orthogonality relation. }  
An orthoframe induces an operation ${(\cdot)}^*$ over subsets $Y \subseteq X$  defined by $Y^* = \{x \in X \mid x \bot Y\} =  \{x \in X \mid x \bot y \text{ for all } y \in Y \}$. Observe  the correspondence to polarity of cones. 
A set $Y \subseteq X$ is called \emph{$\bot$-closed} iff $Y = Y^{**}$. This says that if $x \notin Y$ then there is a $z$ such that  not $x \bot z$, and for all $y \in Y$:  $z \bot y$.

An \emph{orthomodel} for a logic $\mathcal{L}$ over $Fml(P)$ is defined as a structure  $\myI = (X, \bot, (\cdot)^{\myI})$ such that $(X, \bot)$ is an orthoframe and $(\cdot)^{\myI}$ assigns to each $P_{i} \in P$ a $\bot$-closed set over $X$. In a natural way one can extend the assignment function to arbitrary formulae $(A \gbund B)^{\myI} = (A)^{\myI} \cap (B)^{\myI}$, $(\gbnicht A)^{\myI} = ((A)^{\myI})^*$. ($\gboder$ is treated by de Morgan's law).  
The \emph{semantical entailment relation $\vDash$}  then can be defined as  $\myI: A \vDash B$ iff $(A)^{\myI} \subseteq (B)^{\myI}$.  If $U$ is a class of orthoframes, then $A \vDash_{U} B$ means that $\myI: A \vDash B$ for any orthomodel definable in any orthoframe in $U$. 

 In establishing correctness and completeness results w.r.t.\ the orthomodel semantics,  Goldblatt \cite{goldblatt74semantic} constructs for each orthologic $\mathcal{L}$ a canonical model. In order to define this canonical model, he considers maximally consistent sets called $\mathcal{L}$-full sets:  a set $Y$ of formulae is said to be \emph{$\mathcal{L}$-full} iff it is closed w.r.t.\ $\mathcal{L}$-derivability and w.r.t.\  conjunction $\gbund$ and is consistent (i.e., $Y$ is different from the set of all formulae  $Fml(P)$).  The canonical orthomodel is defined as $\myI_{\mathcal{L}} = (X_{\mathcal{L}}, \bot_{\mathcal{L}}, (\cdot)^{\myI_{\mathcal{L}}})$ where $X_{\mathcal{L}}$ consists of all $\mathcal{L}$-full sets and where $X\bot_{\mathcal{L}}Y$ holds iff there is some formula $A$ with $A \in X$ and  $\gbnicht A \in Y$, and where the assignment function $(\cdot)^{\myI_{\mathcal{L}}}$ is defined by $(P_{i})^{\myI_{\mathcal{L}}} = \{Y \in X_{\mathcal{L}} \mid P_{i} \in Y\}$.   
 Goldblatt establishes the following fact:
 
 \begin{proposition}[\cite{goldblatt74semantic}] \label{prop:canonicOrhtoModelCaptures} For any orthologic $\mathcal{L}$: 
 $\Gamma \vdash_{\mathcal{L}} A$ iff $\myI_{\mathcal{L}}: \Gamma \models A$. 
 \end{proposition}
This then gives the completeness and correctness result for the class $\theta$ of all orthoframes for  minimal logic $Omin$.  
 
 \begin{proposition}[\cite{goldblatt74semantic}]\label{prop:minimalModelCompletenessAndSoundness} 
 $\Gamma \vdash_{Omin} A$ iff $ \Gamma \vDash_{\theta} A$.
 \end{proposition}

\section{Towards A Logic of Cones}\label{sect:coneLogic}
We are interested in extensions of minimal orthologic Omin  \cite{goldblatt74semantic} with rules that describe properties of arbitrary ortholattices of cones. For this we consider structures  $(X, \bot)$ where $X = \mathbb{R}^n$ and $\bot$ is a binary relation defined by $\langle x, y \rangle \leq 0$ for $x, y \in X$. 
We call those frames \emph{cone-based orthoframes} because the $\bot$-closed sets in these orthoframes are cones.  \emph{Cone-based orthomodels} assign to each propositional variable a $\bot$-closed set over a cone-based orthoframe. So formally, we aim at extending Omin with rules that hold in all cone-based orthomodels.

 In search for such  propositional logics, we discovered that many well-known candidate rules are falsified by cones.  Hence, first, we are going to discuss counterexamples to those rules before we describe our generalization \cmtp of the orthomodularity rule and prove relevant properties of minimal orthologic extended with \cmtp. In particular, by this line of presentation we intend to show there is no logic of cones yet captured by any of the rules discussed in the literature on logics or lattices. 
 
 \subsection{Cone Counterexamples for Some Prominent Rules\label{sect:counteexamples}}   
Table~\ref{tab:falsifiedRules} summarizes some prominent rules for which falsifying cone-based orthomodels exist. The rules above the double line are rules in a propositional calculus that are either discussed  directly in the logic literature or that are obvious translations of rules discussed in the context of lattice theory. The rules below the double line are defined  without a reference to a calculus but directly with lattice-theoretic notions  because these are not expressible in a propositional calculus without propositional quantifiers---as in the case of the calculus of Goldblatt \cite{goldblatt74semantic} which we consider in this paper. 
    
Counterexamples for (D), (MSD), (wLLJ), (LLJ) (second figure from left), (JSD) (third figure from left) and (W) (outer right figure) are shown in Fig.~\ref{fig:coneCounterexamples}.  The left figure is intended to illustrate the general connection between the rules and their assignments to cones as well the fact that the orthonegation is different from classical negation (complementation).    
We already mentioned that the hexagon lattice (Fig.\ \ref{fig:hexagonAndItsConeEmbedding}) is a counterexample to orthomodularity---and that it is actually the minimal lattice that any lattice not fulfilling orthomodularity must contain as a sublattice. Thus it is also a counterexample for (M) and (D). Moreover, the hexagon lattice serves as a counterexample for the condition of Birkhoff (Bi) and hence also of (MS), (Mac1), (SM): Let $a$ and $b$ of the rule (Bi) be instantiated by $b$ and $\oneg{a}$ of the hexagon in Fig.\ \ref{fig:hexagonAndItsConeEmbedding}. Then the precondition ist fulfilled as $\lbot = b \lmeet \oneg{a} \lcover b$  and $\lbot = b \lmeet \oneg{a} \lcover \oneg{a}$. On the other hand we do not have $b \lcover b \ljoin \oneg{a} = \ltop$ nor  $\oneg{a} \lcover b \ljoin \oneg{a} = \ltop$. 

\begin{table}[tb]
\footnotesize
\centering
\begin{tabular}{{|}l{|}c{|}l{|}}
\hline 
\textbf{Name} & \textbf{Propositional Calculus Rule} & \textbf{Comment} \\ \hline 
Distributivity (D) & 
\begin{tabular}{l}
$A \gbund (B \gboder C) \dashv \vdash$\\
$ (A \gbund B)  \gboder (A \gbund C)$
  \end{tabular} & 
\begin{tabular}{l}
Adding (D)\\
 to ortholattices\\
gives  boolean logic
\end{tabular}\\ \hline 
\begin{tabular}{l}
Meet-Semi-\\
Distributivity (MSD) 
\end{tabular}   &
\begin{tabular}{c}
 $A \gbund C \dashv \vdash B \gbund C$ \\ \hline  
$A \gbund C \dashv \vdash (A \gboder B) \gbund C$
 \end{tabular}
 & \begin{tabular}{l}Weakening of (D)	
	\end{tabular}\\ \hline 
\begin{tabular}{l}
Join-Semi-\\
Distributivity (JSD) 
\end{tabular}   &
\begin{tabular}{c}
 $A \gboder C \dashv \vdash B \gboder C$ \\ \hline  
$A \gboder C \dashv \vdash (A \gbund B) \gboder C$
 \end{tabular}
 & \begin{tabular}{l}Weakening of (D)	
	\end{tabular}\\ \hline 
Modularity (M) &   
\begin{tabular}{c}
 $C \vdash A$\\ \hline
 $A \gbund (B \gboder C) \vdash (A \gbund B) \gboder C$
	\end{tabular} & Weakening of (D)\\	\hline
\begin{tabular}{l}
Orthomodularity\\
(Omr) 
\end{tabular}   &
\begin{tabular}{c}
 $A \vdash B$, $\gbnicht A \vdash C$\\ \hline  
 $A \gboder (B \gbund C) \dashv \vdash (A \gboder B) \gbund (A \gboder C)$
 \end{tabular}
 & \begin{tabular}{l}Weakening of (M);  gives\\
 quantum logic	
	\end{tabular}\\ \hline 
\begin{tabular}{l}Johannson's minimal\\ negation (LLJ)\end{tabular} &\begin{tabular}{c}
 $A \gbund B \vdash C$\\ \hline
	$A \gbund \gbnicht  C \vdash \gbnicht B$
	\end{tabular} & \begin{tabular}{l} Only rule in Dunn's kite \\ \cite{hartonas16reasoning} falsified by cones.\end{tabular}\\ \hline 
\begin{tabular}{l}Weak J.'s negation\\ (wLLJ)\end{tabular} &\begin{tabular}{c}
 $A \gbund B \vdash D \gbund \gbnicht D$\\ \hline
	$A \vdash \gbnicht B$
	\end{tabular} & \begin{tabular}{l} Weakening of (LLJ)\end{tabular} \\ \hline\hline 
\textbf{Name} & \textbf{Lattice rule} & \textbf{Comment} \\ \hline 
M-symmetry \cite[p.\ 66]{stern99semimodular}&  If $M(a,b)$ then $M(b,a)$ & \begin{tabular}{l} Weakening of (M)\end{tabular}\\ \hline 
	\begin{tabular}{l}Mac Lane's Cond.\\
	(Mac$_{1}$) \cite[p.\ 111]{stern99semimodular}
	\end{tabular} &  \begin{tabular}{c} If $b \lmeet c < a < c < b \ljoin a$\\ 
then there is a $d$ with  $b \lmeet c < d \leq b$\\
 and  $a = (a \ljoin d) \lmeet c$\end{tabular}& Weakening of (M)\\ \hline 
	\begin{tabular}{l}Semimodularity\\
	 (SM) \cite[p.\ 2]{stern99semimodular}
	\end{tabular} &  
	If $a\lmeet b\lcover a$ then $b \lcover a \ljoin b$
	 & \begin{tabular}{l}Weakening of (MS) \\ and (Mac$_{1}$)\\ equivalent in the finite \end{tabular}  \\ \hline 
	 \begin{tabular}{l}Birkhoffs's Covering\\
	 (Bi) \cite[p.\ 3]{stern99semimodular} 
	\end{tabular} &  
	If $a\lmeet b\lcover a,b$ then $a,b \lcover a \ljoin b$
	 & \begin{tabular}{l}Weakening of (SM)\\ and (Mac$_{1}$) \end{tabular} \\ \hline 
	\begin{tabular}{l}Whitman Cond. (W)\\
	\cite[p.479]{graetzer11lattice}\end{tabular} & \begin{tabular}{c} If $a \lmeet b \leq c \ljoin d$ \\  then $a \leq c \ljoin d$ or $b \leq c \ljoin d$\\
	 or $a \lmeet b \leq c$ or $a \lmeet b \leq d$ \end{tabular}& \begin{tabular}{l}Used in discussion\\ of free lattices\end{tabular}\\  
	 \hline
\end{tabular}
\caption{Well-known rules excluded from a logic of cones.\label{tab:falsifiedRules}}
\end{table}
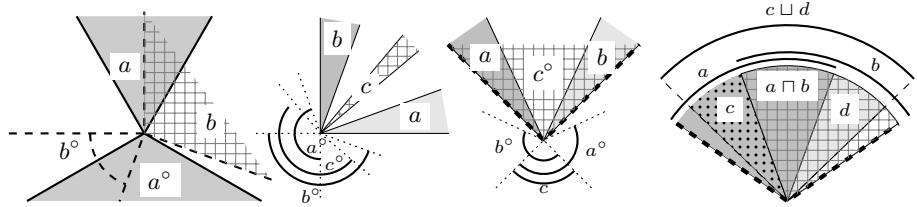
\begin{figure}[tb]
\centering
\begin{tikzpicture}[scale=0.9]
\draw[thick,draw=black,fill=black!20] (60:2) -- (0,0) -- (120:2);
\draw[thick,draw=black,fill=black!20] (330:2) -- (0,0) -- (210:2);
\draw[thick,draw=black,dashed] (340:2) -- (0,0) -- (90:1.8);
\draw[thick,draw=black,dashed] (180:2) -- (0,0) -- (250:1.07);
\node[fill=white,xshift=1mm] at (280:0.7) {$\cnicht{a}$};
\node[fill=white,xshift=-1mm] at (100:1) {$a$};
\begin{scope}
\clip (340:2) -- (0,0) -- (90:1.8) -- cycle;
	\draw[black!50,step=0.15cm] (-1,-1) grid (3,3);
\end{scope}
\node[fill=white] at (10:1) {$b$};
\draw[thick,dashed] ++(180:.8) arc (180:250:.8) node[near start,xshift=-3mm] {$\cnicht{b}$};

\begin{scope}[shift={(2.6,-0.0)},scale=0.95]
\draw[draw=black,fill=black!25] (90:1.8) -- (0,0) -- (70:1.8);
\draw[draw=black,fill=black!10] (0:2) -- (0,0) -- (20:2);
\draw[] (40:2) -- (0,0) -- (50:2);
\begin{scope}
\clip (40:2) -- (0,0) -- (50:2) -- cycle;
	\draw[black!50,step=0.15cm] (-1,-1) grid (3,3);
\end{scope}
\node[fill=white] at (10:1.5) {$a$};
\node[fill=white] at (45:1) {$c$};
\node[fill=white] at (80:1.5) {$b$};
\draw[thick] ++(180:.8) arc (180:340:.8) node[midway,yshift=-1.5mm] {\scriptsize$\cnicht{b}$};
\draw[thick] ++(110:.4) arc (110:270:.4) node[midway,xshift=3mm,yshift=-1mm] {\scriptsize$\cnicht{a}$};
\draw[thick] ++(140:.64) arc (140:320:.64) node[near end,fill=none,xshift=1.5mm,yshift=1.5mm] {\scriptsize$\cnicht{c}$};
\foreach \a in {130, 320,270,110,180,340}
{
  \draw[draw=black,dotted] (\a:1) -- (0,0) -- (\a:1);
}
\end{scope}
\begin{scope}[shift={(5.9,-0.1)}]
\draw[draw=black,fill=black!25] (135:2) -- (0,0) -- (115:2);
\draw[draw=black,fill=black!10] (65:2) -- (0,0) -- (45:2);
\draw[ultra thick, dashed, draw=black] (135:2) -- (0,0) -- (45:2);
\begin{scope}
\clip (45:2) -- (0,0) -- (135:2) -- cycle;
	\draw[black!50,step=0.15cm] (-2,-1) grid (3,3);
\end{scope}
\node[fill=white] at (125:1.5) {$a$};
\node[fill=white] at (55:1.5) {$b$};
\node[fill=white] at (90:1.0) {$\cnicht{c}$};
\draw[thick] ++(225:.64) arc (225:315:.64) node[midway,below,fill=none,yshift=1.0mm] {\scriptsize $c$}; 
\draw[thick] ++(155:.3) arc (155:315:.3) node[near start,left,fill=none] {\scriptsize $\cnicht{b}$};
\draw[thick] ++(225:.5) arc (225:385:.5) node[near end,right,fill=none] {\scriptsize $\cnicht{a}$};
\foreach \a in {225, 25,155,315}
{
  \draw[draw=black,dotted] (\a:1) -- (0,0) -- (\a:1);
}
\end{scope}
\begin{scope}[shift={(9.5,-1.)}]
\begin{scope}
\clip (145:2) -- (0,0) -- (70:2) arc  (70:145:2);
\draw[draw=black,fill=black!25] (145:2.8) -- (0,0) -- (70:2.8);
\end{scope}
\draw[draw=black, dashed] (136:2.5) -- (0,0);

\begin{scope}
\clip (136:2) -- (0,0) -- (110:2) arc  (110:136:2);
\draw[draw=black,pattern=dots] (136:2.8) -- (0,0) -- (110:2.8);
\end{scope}

\begin{scope}
\clip (70:2) -- (0,0) -- (35:2) arc  (35:70:2);
\draw[draw=black,fill=black!10] (70:2.8) -- (0,0) -- (44:2.8);
\end{scope}
\draw[ultra thick, dashed, draw=black] (145:2.) -- (0,0) -- (35:2.);
\draw[draw=black, dashed] (44:2.5) -- (0,0);
\begin{scope}
\clip[draw] (35:2) -- (0,0) -- (110:2) arc (110:35:2);
	\draw[black!50,step=0.15cm] (-1,-1) grid (2,3);
\end{scope}
\node[fill=white] at (122:1.6) {\scriptsize $c$};
\node[fill=white] at (58:1.6) {\scriptsize $d$};
\node[fill=white] at (90:1.7) {\scriptsize $a \cund b$};
\draw[thick] ++(44:2.6) arc (44:136:2.6) node[midway,above,fill=none,yshift=0.1mm] {\scriptsize $c \coder d$}; 
\draw[thick] ++(70:2.1) arc (70:145:2.1) node[pos=0.75,above,fill=none,yshift=0.1mm] {\scriptsize $a$}; 
\draw[thick] (35:2.2) arc (35:110:2.2) node[pos=0.25,above,fill=none,yshift=0.1mm] {\scriptsize $b$}; 

\end{scope}

\end{tikzpicture}
\caption{\label{fig:coneCounterexamples}Left: Cone $a$ and its polar $\cnicht{a}$ represent $A$, respectively $\gbnicht A$. 
Part of $B$ is neither $A$ nor $\gbnicht A$.
Second from left: Counter-example (D), (MSD), (wLLJ), and (LLJ) (letting $C = \lbot$). Third form left: Counter-example (JSD).  Outer right: Conuterexample (W)}
\end{figure}



 As we can find cone-based orthomodels not fulfilling the distributivity law (D) of $\lmeet$ over $\ljoin$,  the logic generating cones must be some non-classical propositional logic. 
 Moreover, some intuitive rewritings known from classical logics are not possible. For example, Johansson's constructive contraposition (LLJ) and its weakening \wllj \raus{(see Sect. \ref{sect:preliminaries})} do not hold.
Closed convex cones falsifying \wllj is a demonstration 
 that there are two different notions of complement.  The first says that $b$ is a complement of $a$ iff $a \lmeet b \lunder \lbot$, the second iff $a \lunder \oneg{b}$. In fact, it is known that any ortholattice which has a unique complement must in fact be distributive and hence must be a Boolean algebra (though this is not the case for arbitrary lattices \cite{dilworth45lattices}).      
Regarding this basic fact one might wonder whether there are at all some interesting implications that can be drawn for closed convex cones in the signature of ortholattices. In Sect.\ \ref{sect:coneLogic} we will show that this is indeed the case.

As the orthomodularity is not fulfilled,  the kind of logic induced by closed cones cannot be that of quantum logics \cite{engesser07new}. In fact, the main underlying geometric structure for quantum logics is not that of a cone (at least in the pioneering work of von Neumann/Birkhoff \cite{birkhoff37logic}) but  that of a closed subspace of a Hilbert space.     
But as the following subsection is going to show, there is a generalization of the orthomodularity rule that holds for any cone-based orthomodel (see Fig.~\ref{fig:rulecmtp}). 

%
%
%
%

The Whitman condition is discussed in the context of free product lattices.  
A counterexample in $\mathbb{R}^2$ is described in the outer right figure in Fig.\  \ref{fig:coneCounterexamples}.

 As we consider convex cones, an obvious question is whether rules considered in work on convex geometries \cite{adaricheva16convex} may hold. In fact, the main property considered in this context is the anti-exchange property (AEP). This rule is defined for  closure systems with a monotonic, idempotent, cumulative operator $cl$. (Hence we did not mention in our overview table.) 
\begin{itemize}
\item[(AEP)] For all $x \neq y \in X$ and cl-closed sets $A$: If $x \in cl(A \cup \{y\})$ and $x \notin  A$, then $y \notin cl(A \cup \{x\})$. 
\end{itemize} 
 The conic hull $\conicHull{\cdot}$ operator is a closure operator, but it does not fulfil (AEP): consider in $\mathbb{R}^2$ the cone $A = \mathbb{R}_{\geq 0} \times \mathbb{R}_{\geq 0}$ and the points $y = (-1,-1)$, $x = (-2,-2)$. Then  $x \in cl(A \cup \{y\})$ just because $x$ is already in the conic hull of $y$, and x $\notin A$. But $y \in cl(A \cup \{x\})$ as $y$ is in the conic hull of $x$.

\subsection{A Logic of Partial Orthomodularity and its Properties}\label{subsect:logicOfPOM}

As cones make up an ortholattice (see preliminaries) our starting point is the minimal orthologic  Omin \cite{goldblatt74semantic}.  
 We introduce the rule  \cmtp  (Fig.\ \ref{fig:rulecmtp}, left) to minimal orthologic and call the resulting logic the \emph{minimal logic of partial modularity}, for short $\pomLogic = Omin + \cmtp$. 



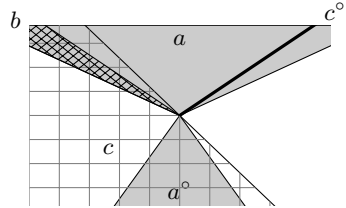
\begin{figure}
\centering
\begin{tikzpicture}[yscale=0.8]
\draw node at (0,0){
 \begin{tabular}{l@{\hskip 0.5cm  }rcl}
\AteilB& $B$ & $\vdash$ & $A$ \\ 
\nichtABunterC& $\gbnicht A \gboder B $ & $\vdash$ & $ C$\\ 
\BCACC& $B \gboder \gbnicht C $ & $\dashv \vdash$ & $ (A  \gbund C) \gboder \gbnicht C$\\ \hline 
\ABconclusion & $A \gbund (\gbnicht A \gboder B)$ & $\vdash$ & $  B$ 
\end{tabular}};
\begin{scope}[shift={(6.2,-0.55)}]
\clip(-2,-1.5) rectangle (2,1.5);
\draw[fill=black!20!white,draw=black] (0,0) -- (150:3cm) --(30:3cm)--cycle; \draw node[black,yshift=10mm] {$a$};
\draw[fill=black!20!white,draw=black] (0,0) -- (240:1.8cm) --(-60:1.8cm)--cycle; \draw node[black,yshift=-10mm] {$a^\circ$};
\draw (130:3cm) -- (310:3cm);
\draw[very thick] (0,0) -- (40:2.5cm); 
\begin{scope}
\clip (130:3cm) -- (310:3cm) -- (-2,-2) -- (-2,2) -- cycle;
\draw[step=0.4cm,black!50!white] (-2,-2) grid (2,2);
\end{scope}
\draw (210:1.1cm) node {$c$};
\begin{scope}
\clip[draw] (0,0) -- (140:4cm) -- (150:4cm) -- cycle;
\draw[step=0.1cm,rotate=45] (-3,-3) grid (3,3);
\end{scope}
\end{scope}
\draw[xshift=6.2cm, yshift= -0.55cm] (144:2.7cm) node {$b$};
\draw[xshift=6.2cm, yshift= -0.55cm] (40:2.7cm) node {$c^\circ$};
\end{tikzpicture}
\caption{Rule \cmtp (left) and its illustration in $\mathbb{R}^2$ (right)}
\label{fig:rulecmtp}
\end{figure}

For the discussion of the rule we consider orthomodels $\myI$ where $(A)^{\myI} = a, (B)^{\myI} = b$ and $(C)^{\myI} = c$.  
The intuition behind this rule is the following (see also Fig.\ \ref{fig:rulecmtp}, right): The conclusion of the rule is exactly the conclusion of the orthomodularity given in its dual form (dsOMr). Without all the premisses this rule does not hold for cones: Considering only the premise \AteilB gives us exactly (dsOMr) which we have shown not to hold for cones. 

But the conclusion of \cmtp can be guaranteed to hold under specific circumstances where $\bCone$ and $\cnicht{\aCone} $ are contained in a halfspace $\cCone$ and where $\bCone$ is at the correct border of $\aCone$  namely when $\bCone$ and $\aCone \cund \cCone$ are \emph{perspective} w.r.t.\ $\cnicht{\cCone}$.\footnote{For the notion of perspectivity see \cite{maeda70theory}}. 

 If $\cCone$ were   larger than a halfspace, it would have to be the whole space $\ctop$ and in that case $\cnicht{\aCone} \coder \bCone \cunder \cCone$ would give no constraint at all.  
In  a simple propositional calculus we cannot express that $\cCone$ is a halfspace. But this is not a problem because  if $\cCone = \ctop$, then the other pre-conditions of \cmtp ensure that the conclusion holds trivially as then $\aCone = \bCone$ and  so $\aCone \cund (\cnicht{\aCone} \coder \bCone) \cunder  \bCone$. 
\begin{theorem}\label{thm:PomHoldsInAllCones}
Any subortholattice of $\mathcal{C}_{n}$ fulfils \cmtp. 
\end{theorem}
\begin{proofsketch}
The idea of the proof is to show that no plane exists in which the cones violate \cmtp. Doing so we generalise the intuition depicted for the 2D case in Fig.\ \ref{fig:rulecmtp} (right) to arbitrary cones.
\end{proofsketch}

One obvious question is whether this rule really constrains the class of  ortholattices, i.e., whether it is not possible to derive  \cmtp from the axioms of minimal orthologic. In other words can one find an ortholattice that fulfils the premises of the rule but not its conclusion?  
Moreover, one might ask whether one can prove a  forbidden-subalgebra theorem which identifies a minimal ortholattice that must be excluded as a subortholattice from any ortholattice that fulfils \cmtp? The answer to both questions ist yes.  
In fact, the ortholattice \minC illustrated on the right of Fig. \ref{fig:hexagonAndItsConeEmbedding} is a lattice falsifying \cmtp.  So, in particular  we can state:  
\begin{proposition}\label{prop:cMTPNotderivable}
Rule \cmtp is not derivable in Omin. 
\end{proposition}

In fact, one can show that this minimal structure \minC must be a sub-ortholattice of any ortholattice iff that ortholattice does not fulfil \cmtp. In other words: 

\begin{theorem}\label{thm:forbiddensubalgebra}
	An ortholattice $L$ fulfils \cmtp iff it excludes \minC. 
\end{theorem}
\begin{proofsketch}
	The idea is to  create the most general lattice fulfilling all orthologic axioms and contains an order $R$ of elements according to the premisses of \cmtp, but not fulfilling the conclusion. 
	Technically, one constructs a free lattice \cite[p.76]{graetzer11lattice} w.r.t.\ the variety of ortholattices and order $R$, i.e., a lattice  which is known to have the universal property: it maps homomorphically into each ortholattice containing $R$.  
\end{proofsketch}
So, we know that \cmtp is a non-trivial rule and that the logic $\pomLogic$ is a real extension of Omin.  What can we say about $\pomLogic$? We give here characterization of $\pomLogic$ along the line of \cite{goldblatt74semantic}. One of the chief results of \cite{goldblatt74semantic} is to show that orthoframes are the adequate structures to define a semantics for which all orthologics are complete. This result is strengthened for the specific case of \emph{orthomodular}  logics, that is logics that extend $Omin$ with an orthomodularity rule (also called \emph{quantum logics}). Goldblatt describes the orthomodularity rule in  a compact form as an axiom
$A \gbund (\gbnicht A \gboder (A \gbund B)) \vdash B$. 
Instead, we use the following equivalent version which fits to the definition of orthomodularity in dual short form (dsOMr). 

\begin{description}
\item[(OM)] If $B \vdash A$, then  $A \gbund (\gbnicht A \gboder B) \vdash B$
\end{description}

We use the notion of relative closure of Goldblatt: $X$ is  \emph{$\bot$-closed in $Z$}  iff: when $x \notin X$ then there is some $y \in Z$ such that  $y \bot X$ but not $x \bot y$.  

Adapting the notion of a quantum frame \cite{goldblatt74semantic} we define   \emph{c-frames} $(X, \bot, \xi)$.

\begin{definition}\label{def:cframe} A \emph{c-frame}  $(X, \bot, \xi)$ 
is a structure where $(X, \bot)$ is an orthoframe and where $\xi$ is a collection of $\bot$-closed subsets  $X$ such that:
\begin{enumerate}
\item $\xi$ is closed under set intersection and the $\bot$-induced polarity operator $Y^* = \{x \mid \text{forall y } \in Y: x \bot y\}$
\item\label{rulepartTwoCframe} For all $a,b,c \in \xi$: If 
\begin{inparaenum}[(i)]
\item $b \subseteq a$,  
\item $c^* \subseteq a \cap b^*$, 
\item  $b^* \cap c \subseteq (a \cap c)^*$, and 
\item $(a \cap c)^* \cap c \subseteq b^*$, 
\end{inparaenum} 
then $b$ ist $\bot$-closed in $a \cap c$. 
\end{enumerate} 
A \emph{c-model} $(X, \bot, \xi, (\cdot)^{\myI})$ is an orthomodel with c-frame  $(X, \bot, \xi)$ and assigning to each $P_{i}$ a set from $\xi$, $(P_{i})^{\myI} \in \xi$.   
\end{definition}

Now, using the same line of reasoning as that of \cite{goldblatt74semantic} for minimal quantum logic (minimal orthologic + orthomodularity rule) one can show that any orthologic $\mathcal{L}_{\cmtp} \supseteq \pomLogic$ that contains the rule \cmtp is  captured semantically by orthomodels over c-frames. 

First we establish a canonical c-model  $\myI_{\cmtp}$. For any formula $A$ and logic $\mathcal{L}$ let $f_{\mathcal{L}}(A)$ be the set of all $\mathcal{L}$-full sets containing $A$.  
The canonical model $\myI_{\cmtp} = (X_{\cmtp}, \bot_{\myI_{\cmtp}}, \xi, (\cdot)^{\myI_{\cmtp}})$ consists of the set  $X_{\cmtp}$  of $\mathcal{L}_{\cmtp}$-full sets  and  the orthogonality relation is defined by $Y \bot_{\myI_{\cmtp}} Z$ iff there is $A \in Y$ and $\gbnicht A \in Z$. The set $\xi$ is  $\xi = \{f_{\mathcal{L}_{\cmtp}}(A) \mid A \in Fml(P)\}$ and the assignment function is $(P_{i})^{\myI_{\cmtp}} = f_{\mathcal{L}_{\cmtp}}(P_{i})$.   

\begin{proposition}\label{prop:myIcMTPIsacModel}
 $\myI_{\cmtp}$ is a c-model.
\end{proposition}

\begin{theorem}\label{thm:pomRepresentation}
$\Gamma \vdash_{\mathcal{L}_{\cmtp}} A$ iff $\Gamma \vDash_{c-frame} A$
\end{theorem}
\begin{proofsketch} The soundness part works by induction on the length of the derivation. 
The completeness part uses Proposition \ref{prop:myIcMTPIsacModel}. 
\end{proofsketch}
%

\section{Related Work\label{sect:relatedWork}}
The logic $\pomLogic$ is a non-distributive propositional logic with an orthonegation. Non-distributive logics are investigated thoroughly in \cite{hartonas16reasoning}. The semantical models considered there are of an abstract kind and not geometrically motivated. Hartonas also considers various other forms of negation, adapting Dunn's kite of negation \cite{dunn96generalized,dunn05negation} to the non-distributive setting.      

We motivated $\pomLogic$  also as a means to investigate knowledge graph embeddings. Highly related work here is that of \cite{gutierrez-basulto18from} and \cite{kulmanov19ele}. Both approaches are orthogonal to ours as they  consider not only unary relations (concepts) but also binary relations \cite{kulmanov19ele}---as required for embeddings of knowledge graphs---or even arbitrary n-ary relations \cite{gutierrez-basulto18from}. But let us note that the general idea underlying cone logic applies also to non-propositional logics, though we did not deal with these in this paper (see our \cite{oezcep20cone,leemhuis20multi-label}).  On the other hand, cone logic provides a full orthonegation---which is not the case for \cite{gutierrez-basulto18from} or for \cite{kulmanov19ele}.     

\section{Conclusion}\label{sect:conclusion}
Investigating the properties of closed cones over  propositional signatures leads to  nontrivial non-distributive propositional logics with an orthonegation.  A case in point is $\pomLogic$  as discussed in this paper. This logic adds to minimal orthologic  further  inferences for scenarios where only partial information on states is available. 
$\pomLogic$ is also interesting for the prospects it offers on connecting KR with machine learning (see \cite{oezcep20cone,leemhuis20multi-label}). 

 $\pomLogic$ is not the end of the story as there are still ortholattices that fulfil $\cmtp$ but are not representable by cones in some $\mathbb{R}^n$. Nonetheless, $\pomLogic$ is a condensation point for ongoing work which leads to stronger rules that seem to provide the necessary bits of constraints in order to characterize cones in the same way as the orthomodularity rule in minimal quantum logic helps characterizing subspaces in Hilbert Spaces.  
 
 \bibliographystyle{splncs04}
 \bibliography{referencesAI2020}
 \appendix
\section{Proofs and Further Results}\label{sect:proofs}

\subsection{Proof of Theorem \ref{thm:PomHoldsInAllCones}}
To prove that all cones fulfill the generalized orthomodularity rule we first give a small lemma. 
The scalar product can be used to measure angles. 
Intuitively, for three vectors $x, y, z$ it holds that $x$ does not belong to the cone generated by $y, z$, if the angle $\angle_x^z$ is larger than $\angle_y^z$.

\begin{lemma}\label{lem:cone-angle}
Let $x, y, z \in \mathbb{R}^n$ with $||x|| = ||y|| = ||z|| = 1$. 
Then:  $\langle x,z \rangle < \langle y,z \rangle$ implies 
$x \not \in \conicHull{\{y,z\}}$. 
\end{lemma}

For every cone $\conicHull{X}$ there exists a family $x_I$ of vectors that generates the cone, i.e., $x \in \conicHull{X}$ iff  $x=\sum_{i\in I} \lambda_i x_i$, with $\lambda_{i} > 0$. 
Assume that the premises of \cmtp hold.
We have to show the existence of some $x \in \aCone \cap \conicHull{\aCone^\circ \cup \bCone}$ with $x\in \aCone\setminus \bCone$ contradicts the premisses of \cmtp. 


Let $X = \aCone \cap \conicHull{\aCone^\circ \cup \bCone} \setminus \bCone$.
Every $x\in X$ can be written as $x=\lambda_x x_a^{\circ} + \mu_x x_b$ with $x_a^{\circ} \in \aCone^\circ$, $x_b \in \bCone$ and $\lambda_x, \mu_x > 0$.
Select any $x\in X$ that maximizes $\frac{\langle x, x_a^{\circ}\rangle}{||x|| \cdot ||x_a^{\circ}||}$.
Since $x\in \aCone$, $x_a^{\circ} \in \aCone^\circ$ it follows that $\langle x, x_a^{\circ}\rangle \leq 0$. 

Case I: $\langle x,x_a^{\circ}\rangle=0$, which intuitively means that there exists an $x \in a$ on the border of $a$ that does not belong to $b$ and is thus located between $\aCone^\circ$ and $\bCone$.
Thus, $\langle x_b, x_a^{\circ}\rangle < 0$ and since $\bCone\subseteq \cCone$ and $\aCone^\circ \subseteq \cCone$ due to \nichtABunterC, it holds $\langle x_b, x_c^{\circ} \rangle > \langle x, x_c^{\circ} \rangle$ for all $x_c^{\circ} \in c^\circ$. 
Therefore, $x\not\in \conicHull{\bCone \cup \cCone^\circ}$ but $x\in \aCone\cap \cCone$, contradicting \BCACC.

Case II: $\langle x,x_a^{\circ}\rangle<0$ in conjunction with \nichtABunterC means that 
$\exists x_{b}\in \bCone : \langle x_{b}, x_a^{\circ}\rangle=0$.
Therefore, $x\not \in \conicHull{\aCone^\circ \cup \bCone}$ which contradicts its construction.
%

\subsection{Proof of Proposition \ref{prop:cMTPNotderivable}}
We have to verify that  MC$_{8}$ falsifies \cmtp. 
We have $a \lmeet c = d$. The premisses of \cmtp are fulfilled, because:  
We have $b \lunder  a$  (premise: $B \vdash A$),  $\oneg{a} \ljoin b = c$ (premise $\gbnicht A \gboder B \vdash C$)     and $b \ljoin \oneg{c} = a = (a \lmeet c) \ljoin \oneg{c}$  
(premise: $ B \gboder \gbnicht C \dashv \vdash (A \gbund C) \gboder \gbnicht C$). 
But the conclusion does not hold as $a \lmeet (\oneg{a} \ljoin b) = a \lmeet c = d \not \lunder b$.  

We want to note that  that MC$_{8}$ is not a lattice appearing only in artificial toy examples  but also in some well-investigated class of ortholattices.  
There is an ortholattice in the form of the extended permutohedron $R(B_{2})$  \cite[p.\ 186]{santocanale14extended}  that does contains MC$_{8}$ and hence does not fulfill \cmtp  (see Fig. \ref{fig:minimalCounterxamplecMTP} with our own labellings). 
\begin{figure}
		\begin{tikzpicture}[xscale=0.5, yscale = 0.7]
		\tikzset{
vertex/.style= {circle,draw=black,fill=black,scale=0.5, thick}
}
\tikzset{hellekante/.style = {color = gray}} 
		\node (bottom) at (4,0){$\lbot$};
		\node (a1) at (1,1){$\oneg{a}$} edge (bottom);  
		\node[vertex] (b0) at (3,1){} edge (bottom); 
		\node (b1) at (5,1){$ \oneg{c}$} edge (bottom); 
		\node[vertex] (a0) at (7,1){} edge (bottom);
		\node[vertex] (c01not) at (0,2){} edge (a1) edge (b0);
		\node (c00not) at (1.4,2){$\oneg{d}$} edge (a1) edge (b1); 
		\node (c11not) at (6.6,2){$b$} edge (b0) edge (a0); 
		\node[vertex] (c10not) at (8,2){} edge (b1) edge (a0);
		\node[vertex] (unot) at (3,3){} edge (b0) edge (b1);
		\node[vertex] (u) at (5,3){} edge[hellekante] (a1) edge[hellekante] (a0);
		\node[vertex] (c10) at (0,4){} edge (c01not);
		\node (c11) at (1.4,4){$\oneg{b}$} edge[hellekante] (c00not); 
		\node (c00) at (6.6,4){$d$} edge[hellekante] (c11not); 
		\node[vertex] (c01) at (8,4){} edge (c10not);	
		\node[vertex] (a0not) at (1,5){} edge (c10) edge[hellekante] (c11) edge[hellekante] (unot);  
		\node (b1not) at (3,5){$c$} edge[hellekante] (c10) edge[hellekante] (u) edge[hellekante] (c00);  
		\node[vertex] (b0not) at (5,5){} edge[hellekante] (c11) edge[hellekante] (u) edge[hellekante] (c01); 
		\node (a1not) at (7,5){$a$} edge (unot) edge[hellekante] (c00) edge (c01); 
		\node (top) at (4,6){$\ltop$} edge (a0not) edge[hellekante] (b1not) edge[hellekante] (b0not)  edge (a1not);		
		\end{tikzpicture}
\hfill
		\begin{tikzpicture}[xscale=0.35, yscale = 0.4]
		\node (top) at (3,10) {$\ltop$};
		\node (C) at (0,7) {$c$} edge (top); 
		\node (nB) at (3,8) {$\oneg{b}$} edge (top); 
		\node (A) at (6,7) {$a$} edge (top); 
		\node (nX) at (3,6) {$\oneg{d}$} edge (nB); 
		\node (nA) at (0,3) {$\oneg{a}$} edge (nX) edge (C); 
		\node (X) at (3,4) {$d$} edge (C) edge (A); 
		\node (nC) at (6,3) {$\oneg{c}$} edge (nX) edge (A); 
		\node (B) at (3,2) {$b$} edge (X); 
		\node (bot) at (3,0) {$\lbot$} edge (nA) edge (B) edge (nC);    
		\end{tikzpicture}
\caption{Left: Ortholattice $R(B_{2})$ Right: Minimal counter example \minC for \cmtp\label{fig:minimalCounterxamplecMTP}}
\end{figure}
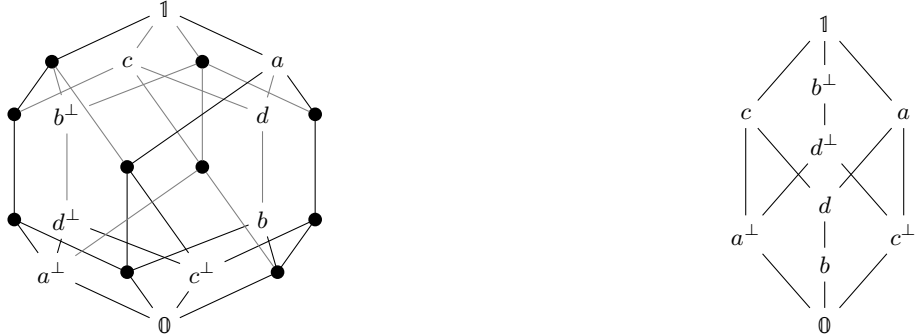

$R(B_{2})$ exemplifies a general construction based on a given transitive binary relation $r$ on some set $D$. The transitive closures of open (i.e., co-transitive in $r$) sets form an ortholattice  called the extended permutohedron. The extended permutohedron of Fig. \ref{fig:minimalCounterxamplecMTP} is induced by the transitive relation over the four-element, diamond-shaped lattice $B_{2}$.

\subsection{Proof of Theorem \ref{thm:forbiddensubalgebra}}
	 Let $L$  fulfill \cmtp. We have to show that $L$ excludes \minC. But we have shown exactly the contraposition for the  permutohedron $R(B_{2})$, in which we argued just with the elements of  \minC.
	 
	Let $L$  not fulfill \cmtp. We have to show that it must contain \minC.
	The conclusion $a \lmeet (\oneg{a} \ljoin b) \not = b$ is only possible when there is a $d$ such that 
	\begin{equation} \label{ABconc2}
	a \lmeet (\oneg{a} \ljoin b)  = d
	\end{equation} and $b < d$, because $b$ is contained in $\oneg{a} \ljoin b$ and because of \AteilB also in $a  \lmeet (\oneg{a} \ljoin b)$.\\
	Now we create the most general lattice that fulfills all orthologic axioms and contains an order $R$ of elements according to the premises of \cmtp as well as  (\ref{ABconc2}). Technically, we construct a free lattice \cite[p.76]{graetzer11lattice} w.r.t.\ the variety of ortholattices and order $R$, i.e., a lattice  which is known to have the universal property: it maps homomorphically into each ortholattice containing $R$.  
	The relations between the lattice elements are as follows:
	\begin{itemize}
		\item $\lbot < b$ (if $b=\lbot$, then $a \lmeet (\oneg{a} \ljoin b) = a \lmeet (\oneg{a} \ljoin \lbot) = a \lmeet \oneg{a} = \lbot = b$: contradiction)
		
		\item $\lbot < \oneg{a}$ (if $\oneg{a} = \lbot$ then $a \lmeet (\oneg{a} \ljoin b) = \ltop \lmeet (\lbot \ljoin b) = b$:  contradiction)
		
		\item $d \leq a \lmeet c$ (with $d = a \lmeet (\oneg{a} \ljoin b)$ (\ref{ABconc2}) and $\oneg{a} \ljoin b \leq c$)
		
		\item $\lbot < \oneg{c}$ (with \BCACC and $\oneg{c} = \lbot$ follows that $b = (a \lmeet c)$, but $d \leq a \lmeet c$, thus $d = b$: contradiction)

		\item $b \lmeet \oneg{a} \leq \lbot$ (with $b  \leq a$ \AteilB)
		
		\item  $b \lmeet \oneg{c} \leq \lbot$ (with 
		\nichtABunterC)
		
		
		
		\item $ \oneg{c} \leq a$ (with \nichtABunterC and ortholattice rule)
		
		
		\item 	$a \lmeet c \leq b \ljoin \oneg{c}$ (with \BCACC)
		
		\item $b \ljoin \oneg{c} \leq a$ (with $b \leq a$ \AteilB and $ \oneg{c} \leq a$)
		
		\item $d \leq \oneg{a} \ljoin b$ (with (\ref{ABconc2}))
		
		\item $\oneg{a} \ljoin b \leq c$ (with \nichtABunterC)
	\end{itemize}	
	The other ortholattice order relations follow from the orthologic calculus.
	
	The next step is to find the distinct elements. If in a homomorphic image two of them are identified, then \cmtp would be fulfilled. 
	It is sufficient to show only one case of $x \leq y$ and $\oneg{y} \leq \oneg{x}$---the other one follows from the definition of orthocomplement. 
	The distinct elements are:
	\begin{itemize}
		\item $\lbot < b$, $\lbot < \oneg{a}$, $\lbot < \oneg{c}$ (see above)
		
		\item $\oneg{a} < b \ljoin \oneg{a}$ (with $b \leq a$ \AteilB and $a,b \not = \lbot$)
		\item  $\oneg{c} < b \ljoin \oneg{c}$ (analogue)
		
		
		\item $a \lmeet c < c$ (with \nichtABunterC we have  $\oneg{a} \leq c$ and $\oneg{a} \not = \lbot$)
		
		\item $a \lmeet c < b \ljoin \oneg{c}$ ($a \lmeet c = b \ljoin \oneg{c}$ only possible, if $\oneg{c}  = \lbot$ or $\oneg{c} \leq a \lmeet c$ (else \BCACC is not fulfilled) both not possible)
	\end{itemize}
	A homomorphic image of the free lattice, where the distinct elements $a,b,c,d$ are not identified, results in \minC.
	Thus every ortholattice which does not fulfill \cmtp must contain \minC.

\subsection{Proof of Proposition \ref{prop:myIcMTPIsacModel}}
As in \cite{goldblatt74semantic} one checks that   $\bot_{\myI_{\cmtp}}$  is indeed an orthogonality relation 
and that the sets in $X_{\cmtp}$---due to being $\mathcal{L}_{\cmtp}$-full---are also $\bot_{\myI_{\cmtp}}$-closed.  We have to show that $\xi$ fulfils the conditions of a c-frame. 
All elements in $X_{\cmtp}$ are closed under intersection and the operator $\cdot^*$ induced by $\bot_{\myI_{\cmtp}}$ because $\xi$ contains assignments $(\cdot)^{\myI_{\cmtp}}$ for all formulae (wich are closed under $\gbund$ and $\gbnicht$).  So we have to show that   Def. \ref{def:cframe}.\ref{rulepartTwoCframe} holds. But this is also clear because each sub-condition defined for the sets in $\xi$  directly translates to a corresponding premise in the rule \cmtp. For example, let $b, a \in \xi$ and $b \subseteq a$. Assume that $b = (B)^{\myI_{\cmtp}}$, so the set $b$ is the denotation of the formula $B$ and similarly $a = (A)^{\myI_{\cmtp}}$. Then according to Prop.\  \ref{prop:canonicOrhtoModelCaptures}  this means that $ B \vdash_{\mathcal{L}_{\cmtp}} A$. Similarly one derives that the other premises must hold and so the conclusion  $A \gbund (\gbnicht A \gboder B) \vdash_{\mathcal{L}_{\cmtp}} B$ must hold. Using again Prop. \ \ref{prop:canonicOrhtoModelCaptures} this means that   $(A \gbund (\gbnicht A \gboder B))^{\myI_{\cmtp}} \subseteq (B)^{\myI_{\cmtp}}$. 
Therefore,  $B$ is $\bot_{\myI_{\cmtp}}$-closed. 

\subsection{Proof of Theorem \ref{thm:pomRepresentation}}
``$\rightarrow$'' (soundness):  This works by induction on the length of the derivation considering each rule case by case.  The cases for 
the  rules of $O_{min}$ are clear (see Goldblatt's soundness proof for $O_{min}$ \cite[p. 26]{goldblatt74semantic}).  The only new part is the rule \cmtp. So assume that $\myI$ is a c-model making the premises of \cmtp true. But this means actually due to Def. \ref{def:cframe}.\ref{rulepartTwoCframe} that $(B)^{\myI}$ is $\bot$-closed in $(A)^{\myI} \cap (C)^{\myI}$. If the conclusion of \cmtp would not hold then we would have an  $x$ such that $x \in (A \gbund (\gbnicht A \gboder B))^{\myI}$ from which it follows that $x \in (A \gbund C )^{\myI}$  and $x \notin (B)^{\myI}$.  From the $\bot$-closure of $(B)^{\myI}$ in $(A)^{\myI} \cap (C)^{\myI}$ we know actually that there must be some $y \in A$ such that $y \in     ((B)^{\myI})^*$ and not $x \bot y$.  But as  $x \in (A \gbund \gbnicht B)^{\myI}$ we have in particular $x \in (\gbnicht B)^{\myI}$ so $x \bot y$, which gives a contradiction. 

``$\leftarrow$'' (completeness): 
Let $\Gamma \not\vdash_{\mathcal{L}_{\cmtp}} A$. According to Prop.\  \ref{prop:myIcMTPIsacModel},  structure $\myI_{\mathcal{L}_{\cmtp}}$ is a c-model. Moreover,  $\myI_{\mathcal{L}_{\cmtp}}$ is a canonical model of $\mathcal{L}_{\cmtp}$ and so we can use  Prop.\ \ref{prop:canonicOrhtoModelCaptures} to infer that $\myI_{\cmtp}: \Gamma \not\vDash A$ which entails that the desired relation  $\Gamma \not\vDash_{c-frame} A$ holds.

\subsection{Towards Reintroducing (Weakenings) of Distributivity }
In this section we are going to discuss how to change the cone configuration lattices in order to regain (weakenings) of distributivity. 
We first discuss how to regain distributivity for cones by considering specific alignments of cones. Then we consider the case where  full $C_{n}$ (not arbitrary sublattices) are considered and discuss necessary and sufficient conditions for M-symmetry.   

\subsubsection{Cone Configuration for Distributivity and Al-Cones} 
The lattice induced by cones can be restricted to distributivity by considering only \emph{axis-aligned cones} \cite{oezcep20cone}.  
\begin{definition}
	An \emph{axis-aligned cone (al-cone)} in the $n$-dimensional space is of the form	
	\begin{equation}\label{eq:alcone}
	(\aCone_1, ... ,\aCone_n) \text{ where each } \aCone_i \in \{\mathbb{R},\mathbb{R}_+,\mathbb{R}_-,\{0\}\}.
	\end{equation}
\end{definition}
The class of al-cones fulfills the distributivity property and hence can be applied for  classical logics $\mathcal{L}_{conc}$ that provide propositional concept constructors (concept-and, concept-negation, concept-or). The class of al-cones is useful as al-cones have simple encodings, but it is not the only distributive subclass of closed convex cones. In fact we have the following characterisation.  
\begin{proposition}
	A subclass $\mathcal{C}_{n}'$ of closed convex cones fulfills distributivity iff for each combination $\aCone$, $\bCone \in \mathcal{C}_{n}'$ it holds that $\aCone \cund \bCone \not\cunder \cbot $ or for each element  $x_a \in \aCone$ and each element $ x_b \in \bCone$ the scalar product $\langle x_a,x_b\rangle \leq 0$. 
\end{proposition}

\begin{proof}
		We use  MacNeille's axiomatization of Boolean algebras according to \cite[axiom B67, p.\ 114]{padmanabhan08axioms}. 
		All but one axiom of this definition follow from the definition of orthologic and the geometric interpretation of cones and are fulfilled by all cone structures. This reduces the proof to show the validity of the rule \wllj (if $a \wedge b \leq \lbot$, then $a \leq \oneg{b}$).
		
		``$\rightarrow$'': When $\aCone$ and $\bCone$ are disjoint, then $\aCone$ must be a subset of $\oneg{\bCone}$. By the definition of negation as polar cone this is possible  only if for each element  $x_a \in \aCone$ and each element $ x_b \in \bCone$ the scalar product $\langle x_a,x_b\rangle \leq 0$.  When $\aCone$ and $\bCone$ intersect, then the condition is fulfilled trivially.
		
		``$\leftarrow$'': Follows from \wllj and the definition of polarity.
\end{proof}

\subsubsection{M-Symmetry for $\mathcal{C}_{n}$}
In the paper we consider arbitrary ortholattices that are sublattices of $C_{n}$. When considering only  $C_{n}$ (for arbitrary $n$) one can draw further consequences. 
 
An \emph{atom} $b$ in such a lattice is an element that covers  $\lbot$.
 A lattice is called \emph{atomistic} if every element is the  (not necessarily finite) supremum of atomic elements.  It is called \emph{atomic} iff each element has an atom below it.  
It is clear that the class $\mathcal{C}_{n}$ is atomistic because the one-dimensional  cones $x_{\mathbb{R}_{\geq 0}} = \{\lambda x \mid \lambda \in \mathbb{R}_{\geq 0}\}$ for  $x \in \mathbb{R}^n$  (rays in positive direction of $x$ emanating from the origin) are atoms and every cone is generated by the conic hull of such cones.  
Moreover, the lattice is dual atomic, i.e., there are cones, called dual-atoms \cite{maeda70theory}  that are covered by the top cone $\ctop = \mathbb{R}^n$. These cones are just the ones of the form $\XCone^{\circ}$ for all atoms $\XCone$ (rays) in the lattice.

In this section we give some positive results on the lattice-theoretic properties of the class of (all) cones in a Euclidean space. We mainly deal with the question under which conditions we have a modular pair of cones and when this pair commutes (M-symmetry).

Our first result is a sufficient condition for modular pairs: if the conic hull of the union of $\aCone,\bCone$ is already the set union of $\aCone$ and $\bCone$ then $(\aCone,\bCone)$ is a modular pair.

\begin{proposition}\label{prop:sufficientConditionModularity}
For all $\aCone, \bCone$ being closed convex cones:   If $\conicHull{\aCone \cup \bCone} = \aCone \cup \bCone$ then $M(\aCone,\bCone)$ (and hence also $M(\bCone,\aCone)$).   
\end{proposition}
\begin{proof}
Assume that $\conicHull{\aCone \cup \bCone} = \aCone \cup \bCone$. We have to show that $M(\aCone,\bCone)$ for which it is sufficient to show that for all $\cCone \cunder \aCone$: $\aCone \cund (  \bCone \coder \cCone ) \cunder (\aCone \cund \bCone) \coder \cCone$.  Let $x \in \aCone \cund (  \bCone \coder \cCone )$. In particular $x \in \bCone \coder \cCone$, hence there are $v \in \bCone$ and $w \in \cCone \cunder \aCone$ with $x = v + w$. If $x \in \bCone$, then we are already done as then $x \in \aCone \cund \bCone$.  So assume that $x \notin \bCone$. Consider the function $f(\lambda) = v + \lambda w$ for all $\lambda \geq 0$. The function $f$ is continuous. We have $f(0) = v \in \bCone$ and $f(1) = v +w = x\in \aCone \setminus \bCone$. 

Moreover, we have $\mu_{1}v +\mu_{2} w \in \bCone \coder \aCone = \aCone \coder \bCone = \aCone \cup \bCone$ for all $\mu_{1}, \mu_{2} \geq 0$. In particular $f(\lambda) \in \aCone \cup \bCone$  for all $\lambda \geq 0$ (choosing $\mu_{1} = 1$ and $\mu_{2}  =\lambda$).  So, as both $\aCone$ and $\bCone$ are assumed to be closed there must be some point $\lambda' \in [0,1]$ such that $f(\lambda') = v + \lambda'w \in \aCone \cund \bCone$. Then $x = x +w  = x + \lambda' w + (1- \lambda') w = f(\lambda') +  (1- \lambda') w$. 
As $f(\lambda') \in \aCone \cund \bCone$ and $(1- \lambda') w \in \cCone$ (because $\cCone$ is a cone and $(1- \lambda') \geq 0$) we have $x \in (\aCone \cund \bCone) \coder \cCone$. 
  \end{proof}


Unfortunately, this property relies on set union and not just on the order-theoretic notion $\cunder$. 
We show that this supposed problem can be mitigated by giving an equivalent description of the sufficient condition based only on the order. 

\begin{proposition}\label{prop:charakterisierungAbschlussVereinigung}
 $\conicHull{\aCone \cup \bCone} = \aCone \cup \bCone$ iff there is no $\cCone \neq \cbot$ such that 
 \begin{inparaenum}
 \item $\aCone \cund \cCone \cunder \cbot$, 
 \item $\bCone \cund \cCone \cunder \cbot$, and 
 \item $\aCone \coder \bCone \coder \cCone = \aCone \coder \bCone$. 
 \end{inparaenum}
\end{proposition} 
\begin{proof}
``$\rightarrow$'':  We show the contraposition. So assume there is a $\cCone$ such that 
$\aCone \cund \cCone \cunder \cbot$, $\bCone \cund \cCone \cunder \cbot$, and 
 $\aCone \coder \bCone \coder \cCone = \aCone \coder \bCone$. From the first two items it follows that $\cCone \cap \aCone = \emptyset$ as well as $\cCone \cap \bCone = \emptyset$. Hence $\aCone \cup \bCone \cup \cCone \supsetneq \aCone \cup \bCone$ (*) holds. 
We have to show  that  $\conicHull{\aCone \cup \bCone} \neq \aCone \cup \bCone$. Assume for contradiction that $\conicHull{\aCone \cup \bCone} = \aCone \cup \bCone$. Then 
$\aCone \cup \bCone \cup \cCone \subseteq$ $\aCone \coder \bCone \coder \cCone =$ $\aCone \coder \bCone =$ $\aCone \cup \bCone$, 
which contradicts (*). 

``$\leftarrow$'':
We show again the contraposition. Assume that $\conicHull{\aCone \cup \bCone} \neq \aCone \cup \bCone$, i.e., $\conicHull{\aCone \cup \bCone} \supsetneq \aCone \cup \bCone$. Let
 $z \in \conicHull{\aCone \cup \bCone} \setminus \aCone \cup \bCone$ (**).  
 We must have $z \neq 
 0$. Consider $\cCone := \conicHull{\{z\}} = \{\lambda z \mid \lambda \geq 0\}$. We must show that $\cCone$ fulfils all three properties. Ad 1.: $\cCone \cund \aCone \cunder \cbot$ holds because otherwise there would be some $\lambda_{0} > 0$ with $\lambda_{0} z \in \aCone$. But as $\aCone$ is a closed convex cone also $z = \frac{1}{\lambda_{0}} \lambda_{0} z \in \aCone$ , contradicting $(**)$. In a similar way one shows that  $\cCone \cund \bCone \cunder \cbot$ holds. The third property holds because $z \in \conicHull{a \cup b}$. 
\end{proof}
We denote the property expressed in  Prop.\ \ref{prop:charakterisierungAbschlussVereinigung} as $V(a,b)$ (mnemonics:  ``there exists  no $c$ that can be added \textbf{v}acuously to $a,b$''.) 
With this notation we get as a corollary a sufficient condition for $M(a,b)$  that is based on lattice-theoretic notions only.   

\begin{corollary}\label{cor:suffcientConditionsLatticeStyle} For all $\aCone, \bCone$ being convex, closed cones:   If $V(\aCone,\bCone)$ then $M(\aCone,\bCone)$ (and hence also $M(\bCone,\aCone)$).   
\end{corollary}

One might wonder whether replacing the condition $\conicHull{\aCone \cup \bCone} = \aCone \cup \bCone$ by $V(\aCone,\bCone)$ could have lead to a loss of information so that the sufficiency condition holds for arbitrary orthocomplemented lattices and not only for those generated by closed convex cones. If this were the case then the sufficiency condition would not at all contribute to characterizing  the structure of lattices induced by convex cones. 
But as the hexagon lattice in Fig. \ref{fig:hexagonAndItsConeEmbedding} shows
this is not the case.

The opposite direction in Prop.\ \ref{prop:sufficientConditionModularity}, namely   that  $ M(\aCone,\bCone)$ entails $\conicHull{\aCone \cup \bCone} = \aCone \cup \bCone$, does not hold necessarily. Consider a cone $\aCone$ that is a ray  in $\mathbb{R}^2$, say $ \{(x,x) \mid x \geq 0$\} and  a solid cone $\bCone$ with $\bCone \cund \aCone = \cbot$, say $\mathbb{R}_{+} \times \mathbb{R}_{-}$ 
Then for $\aCone$ we do not have any proper subcone $\cCone \cunder \aCone$, so that $M(\aCone,\bCone)$ must hold trivially: for $\cCone=\aCone$ we have $\aCone \cund (\bCone \coder \cCone) = \aCone \cund (\bCone \coder \aCone) = \aCone = \cbot \coder \aCone = (\aCone \cund \bCone) \coder \aCone = (\aCone \cund \bCone) \coder \cCone$. But on the other hand $\conicHull{\aCone \cup \bCone}$ is the area  $\mathbb{R}_{+} \times \mathbb{R}_{-} \cup \{(x,y) \mid x \geq y \geq 0\}$
 and this is not  the same as  $\aCone \cup \bCone$.  

%
%
%
%
%
%
%
%
%

  In fact, this example shows that the set of closed convex cones  is not M-symmetric: As already shown, $M(\aCone,\bCone)$ holds trivially. 
  But $M(\bCone,\aCone)$ does not hold in this case.  That is, one  can  find $\cCone \cunder \bCone$ such that $\bCone \cund (\aCone \coder \cCone) \neq (\bCone \cund \aCone) \coder \cCone$: take any cone $\cCone$ properly contained in $\bCone$ (which is possible as $\bCone$ is solid). Then $(\bCone \cund \aCone) \coder \cCone = \cCone$ whereas $\bCone \cund (\aCone \coder \cCone)$ is larger than $\cCone$.   
The problem here is that we chose $\aCone$ as a non-solid cone. 
Restricting the consideration to solid objects $a$ leads to the following necessary condition for modular pairs.

\begin{proposition} \label{prop:necessaryConditionModularity}
For all  closed convex solid cones $\aCone$ and closed convex (not necessarily solid) cones $\bCone$ with $\aCone \cap \bCone = \emptyset$:   if $M(\aCone,\bCone)$ then  $\conicHull{\aCone \cup \bCone} = \aCone \cup \bCone$.

\end{proposition}
\begin{proof}
We are going to show the  contraposition. For this purpose  assume that $\conicHull{\aCone \cup \bCone} \neq \aCone \cup \bCone$. We have to show that $M(\aCone,\bCone)$ does not hold.   
As the cone $\aCone$ is solid  it has a non-empty interior $int(\aCone)$. So let $v \in int(\aCone)$ and  let $\cCone$ be the ray $v\mathbb{R}_{\geq 0}$. 
Then $\aCone \cund(\bCone \coder \cCone) \supsetneq (\aCone \cund \bCone) \coder \cCone = \cCone$ because  $\bCone \coder \cCone$ contains a convex  solid cone $\cCone'$ with  $\aCone \supseteq \cCone' \supsetneq \cCone$. 
\end{proof}

\end{document}